\documentclass[review]{elsarticle}

\usepackage{lineno,hyperref}
\usepackage{amsmath, amsfonts, amsthm, amssymb}
\usepackage{graphicx}
\usepackage{subfig}
\usepackage{hyperref}
\usepackage{algorithm}
\usepackage{algorithmic}
\usepackage{rotating}
\usepackage[shortlabels]{enumitem}

\theoremstyle{definition}

\newtheorem{theorem}{Theorem}[section]

\newcommand{\pr}{\text{Pr}}

\modulolinenumbers[5]

\journal{Decision Support Systems}

\begin{document}

\begin{frontmatter}

\title{Discriminative Data-driven Self-adaptive Fraud Control Decision System with Incomplete Information}

%\author{Junxuan Li\fnref{myfootnote}}

\author[mysecondaryaddress]{Junxuan Li}
\ead{junxuan.li@gatech.edu}

\author[mysecondaryaddress]{Yung-wen Liu}
\ead{yungliu@microsoft.com}

\author[mysecondaryaddress]{Yuting Jia}
\ead{yutjia@microsoft.com}

\author[mysecondaryaddress]{Jay Nanduri}
\ead{jayna@microsoft.com}

\address[mysecondaryaddress]{Dynamic 365 Fraud Protection, Microsoft, Redmond, WA 98052}

\begin{abstract}
While E-commerce has been growing explosively and online shopping has become popular and even dominant in the present era, online transaction fraud control has drawn considerable attention in business practice and academic research. Conventional fraud control considers mainly the interactions of two major involved decision parties, i.e. merchants and fraudsters, to make fraud classification decisions without paying much attention to dynamic looping effect arose from the decisions made by other profit-related parties. This paper proposes a novel fraud control framework that can quantify interactive effects of decisions made by different parties and can adjust fraud control strategies using data analytics, artificial intelligence, and dynamic optimization techniques. Three control models, Naive, Myopic and Prospective Controls, were developed based on the availability of data attributes and levels of label maturity. The proposed models are purely data-driven and self-adaptive in a real-time manner. The field test on Microsoft real online transaction data suggested that new systems could sizably improve the company's profit.\\
\end{abstract}


\end{frontmatter}

%\linenumbers
\section{Introduction}
As E-commerce has grown explosively in recent years, many merchants have been providing some centralized platforms for consumers to buy products with ''One-Click". Although online (card-not-present) type of transactions have offered the great benefit of consumer convenience, it also has increased the high risk of transaction frauds.  As a result, merchants unavoidably have to employ many resources to develop an effective and efficient mechanism for fraud detection and transaction risk control. These control systems usually consist of two core engines: a risk scoring engine and a risk control engine.

The risk scoring engine is designed to measure the risk level of each transaction. Instead of assigning a transaction with explicit 0-1 (legitimacy - fraud) classification, the majority of merchants calculate the risk score for each transaction based on its attributes, such as purchase price, order quantity, payment information, product market, etc. Whenever a transaction with a higher score is seen, it is more likely to be fraudulent. With the help of big data and machine learning technologies, the modern scoring model has been significantly improved using streaming historical data. 

The risk control engine gets involved once a risk score is calculated. Some transactions that violate predetermined policies or rules get instantly rejected. These predetermined rules and policies are set due to some governments and merchants made regulations, or they are needed when some obvious frauds require immediate blockade. However, the majority of frauds fail to be restrained by these rules, so the risk control engine needs to step in and further prevent more fraudulent transactions using the risk scores. Conventional risk controls apply static risk cut-off score thresholds: approve transactions with risk scores lower than the low score threshold; reject transactions with scores higher than the high score threshold; utilize human intelligence (manual review) for further investigations on transactions with the risk scores in-between. The cut-off score thresholds are set so that the inline fraud detection system can optimally prevent fraudsters' attacks. This threshold band method is widely applied in e-commerce merchants and financial institutions. Despite the fact that the method of risk score evaluation has been significantly improved during the past few years, due to the following three main reasons decisions made by risk scores are still not always reliable: 1) Rapid changes in fraudsters' behavior patterns; 2) Loss of fraud signals from rejected transactions, and; 3) Long data maturity lead time.  Because of these issues, the conventional fraud control engine lacks for flexibility and capability of real-time self-adjustment, and hence cannot always provide the most accurate risk decisions.

Our research motivation for this paper stemmed not only from the drawbacks of the current fraud control systems but also from the broader view of various risk control parties who contribute to the final decisions in different transaction flows. Merchants' risk control decision making should not be isolated from the entire decision environment, where payment issuing banks and manual review teams make follow-up decisions that constitute the final decisions on every transaction.  Figure \ref{fig:flow} depicts how a transaction is processed through different decision stations until it reaches its final decision.

\begin{figure}[hbtp]
\centering
\includegraphics[scale=0.5]{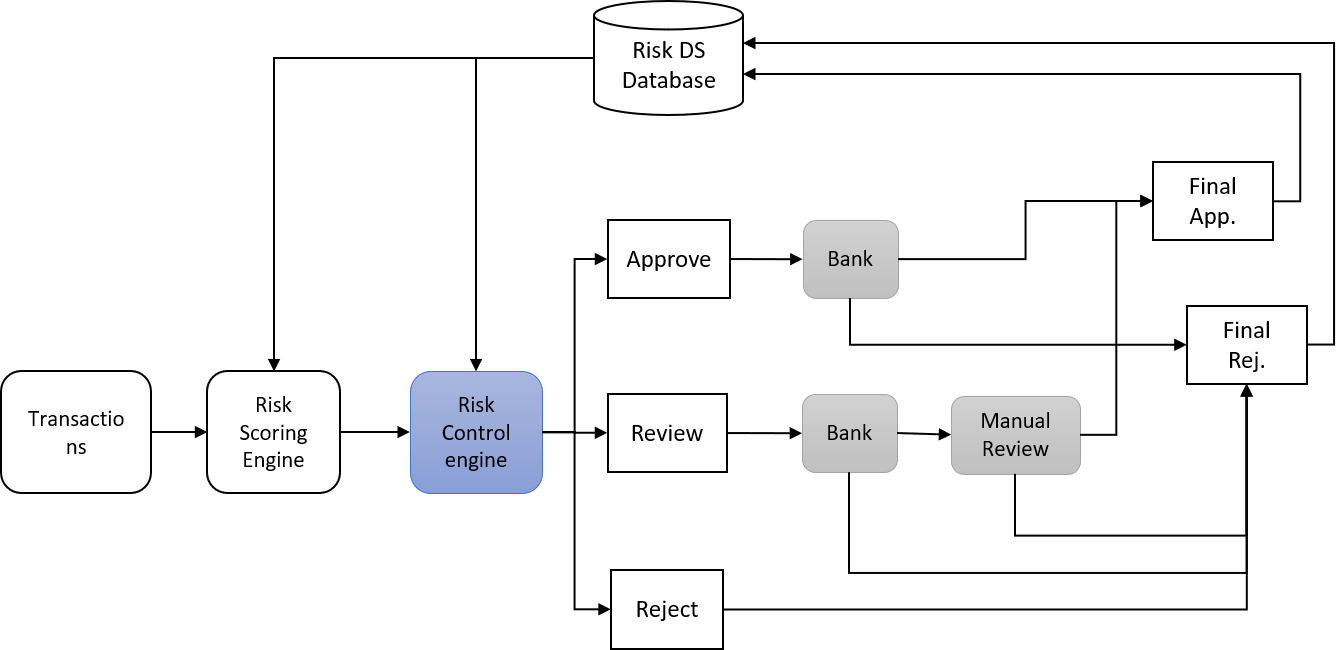}
\caption{Transaction flow demonstration}
\label{fig:flow}
\end{figure}

When a transaction arrives, the risk scoring engine calculates its risk level score based on all its associated features.The risk control engine then makes a decision (approval, rejection, MR review) using some important attributes of this transaction (including its risk score). If the transaction is approved by the risk control engine, it is then sent to the bank for the follow-up decision (a bank authorized transaction is marked as \textit{Final Approval}, and a bank declined transaction is marked as \textit{Final Rejection}).  If the transaction is rejected by the risk control engine, it is directly marked as {\it{Final Rejection}}. If the transaction is not approved nor rejected by the risk control engine, it would also be sent to the bank first.  Only if the bank authorizes the transaction, it has the chance to reach to the manual review (MR) agents for further investigation and for its final decision (a transaction that is authorized by bank and approved is marked as \textit{Final Approval}, and marked \textit{Final Rejection} otherwise).  The blue box indicates the target of this research, and the grey boxes point out other involved decision-making parties. 

Banks are regarded as a single decision party for simplicity. From the data, we found that when the risk control engine approved and submitted transactions that included more frauds (false negative: wrongful approval) to the banks, when banks sensed it, they became more conservative and would decree more rejections of good transactions (false positive: wrongful rejection).  Data also showed that when the risk control engine submitted transactions that included fewer frauds (true negative: rightful approval) to the manual review (MR) teams,  manual review teams tended to have much harder time to make accurate risk decisions since fraud patterns are less massive and recognizable. Interactions of different decision parties, legitimate customers and fraudsters are demonstrated in Figure \ref{fig:interaction}.

\begin{figure}[htbp]
\centering
\includegraphics[scale=0.42]{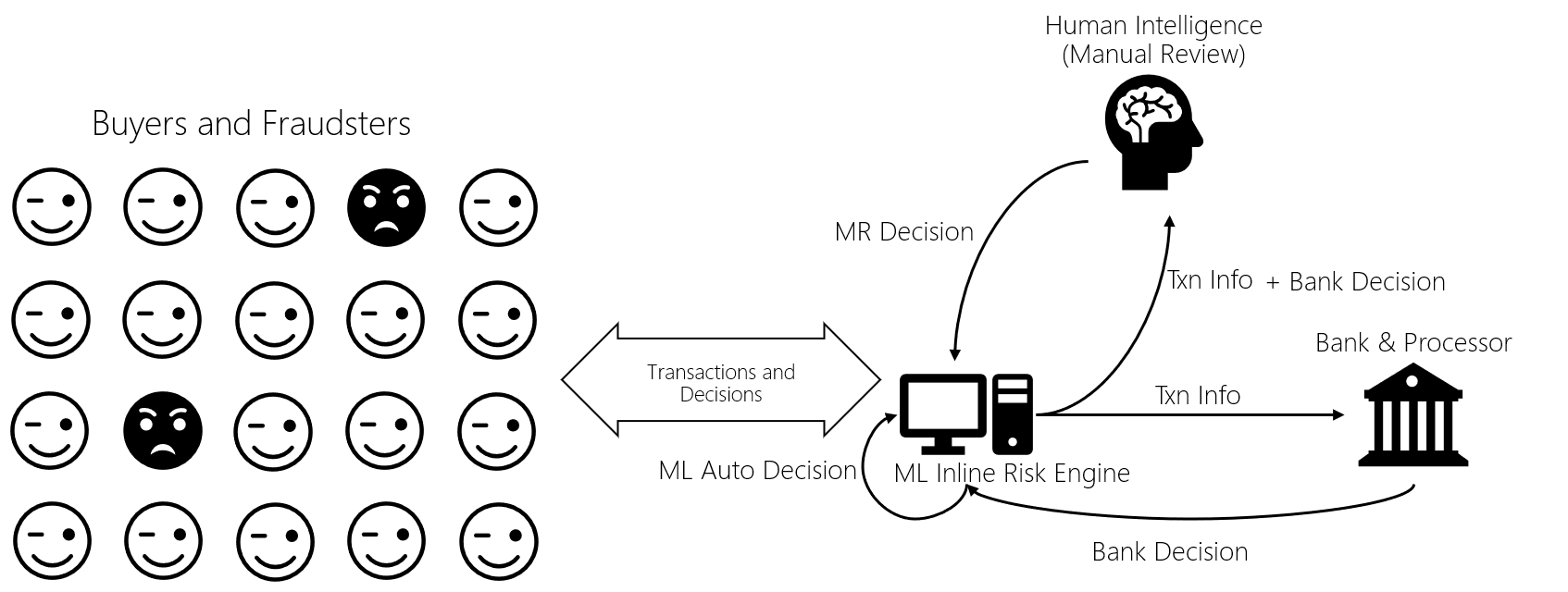}
\caption{Demonstration of interactions among decision parties and buyers/fraudsters}
\label{fig:interaction}
\end{figure}

Considering the high total dollar amount of e-commerce transactions taking place in this such rapidly changing risk decision environment, there is a strong need to design a fraud control engine that can conquer all the aforementioned challenges and optimize the decision accuracy so that the higher profit can be reached. In this paper, the proposed control framework is designed to achieve the following:
\begin{enumerate}[(i)]
\item Adaptive learning: the proposed risk control engine is trained using streaming transaction records which might include some incomplete information such as the immature label, and it can adaptively recognize the new decision environment;
\item Discriminative control: instead of using static uniform cut-off thresholds, the proposed control system can assign inline decision (Approve, Reject or Manual Review) in a real-time manner based on the attributes of each incoming transaction;
\item Data-driven: the risk control is entirely data-driven which helps avoid unreliable ad hoc human-made hard-coding rules on risk decisions.
\end{enumerate}
The field test on Microsoft real online transaction data suggested that the proposed control system could significantly improve the company's profit by reducing the loss caused by inaccurate decisions (including both wrongful approvals and wrongful rejections).

The rest of this paper is structured as follows:  In Section \ref{sec:review} previous research work related to fraud control is first outlined and the existence of the research gap is discussed. In Section \ref{sec:problem} the {\it{Perfect State Dynamic Model}} with rigorous mathematical formulation are introduced and the intractableness of the model is then discussed. Three approximate dynamic control models are proposed in Section \ref{sec:control}, and the test results of their performance are included in Section \ref{sec:test}. Section \ref{sec:conclusion} concludes this paper.

\section{Related Research}\label{sec:review}
Online shopping fraud detection research using machine learning methodologies started from early 90's right after the occurrence of E-commerce, in which the major research task was to evaluate fraud risk levels of transactions. Fraud risk level was measured using risk scores, and thus the research on risk scoring gained widespread attention. These scoring engines were inspired by neural network \cite{GhoshReilly1994FraudDetectionNN,Aleskerov1997FraudDetectionNN,Dorronsoro1997FraudDetectionNN}, decision tree \cite{mena2002FraudDetectionBook,Kirkos2007FraudDetectionDT,Sahin2013FraudDetectionDT}, random forest \cite{Bhattacharyya2011FraudDetectionRF}, network approach \cite{APATE2015FraudDetectionNW} and deep neural network \cite{Kang2016FraudDetectionCNN}. Readers who are interested in this topic may also refer to \cite{Review2011DSS} and references therein for other related papers that discussed different scoring methods. Despite the fact that current research admits the fact that fraud patterns keep changing and fraud risk scores are not always that reliable, no existing papers discuss how to optimally utilize these scores in fraud control operations. On the other hand, data mining papers provide weak guidance in detailed operations, as risk score is indeed a blur expression of fraud. There is currently no literature demonstrating how to deal with the transactions in "gray zone", where the risk score of a transaction is neither too low nor too high. Additionally, no literature has addressed interactions of decisions made by multiple parties for transaction risk control. The main reason of lack of related literature is that e-commerce data are strictly confidential and thus very limited access are granted for academic researches. Our this paper fills the gaps between the transaction fraud evaluation and the systematic risk control operations.

Dynamic control research started from the 1940's.  We suggest \cite{MDP1994} and \cite{DP1995} for comprehensive introduction to dynamic optimal control methods, as well as their applications in communication, inventory control, production planning, quality control, etc.. In this research, we investigated an important segment of the dynamic control research, dynamic optimal control with incomplete information, as the main technical foundation of our paper which targeted the challenge of some fraud control systems that can only obtain and utilize partially mature data for modeling. One previous related research is Partially Observed Markovian Decision Process (POMDP). POMDP is a sequential decision-making model that deals with inaccurate and incomplete observations of the system state or decision environment. It models/infers transition probability matrix, and the underlying relationship between partially observed and true states (fully observed) information. However, POMDP brings in the significant computational challenges and often requires carefully designed heuristic algorithms to achieve sub-optimal solutions. Structural properties of the reward function and computational algorithms of POMDP are available in \cite{POMDP1973Finite}, \cite{POMDP1978InfiniteDiscount}, \cite{POMDP1980Platzman}, \cite{POMDP1989SolutionProcedureWhite}, \cite{POMDP1991WhiteSurvey}, \cite{POMDP1994AAAI94}, \cite{POMDP1998HeuristicWhite}, \cite{POMDP2004HeuriticWhite} and \cite{POMDP2010Geometric}. Another related research in dynamic control with incomplete information is Adaptive Dynamic Programming (ADP). ADP assumes that perfect information is not known a priori and needs to be gradually learned from historical data or feedback signals of the dynamic system. ADP concepts started from 1970's and contributed as one of the core methods in reinforcement learning. In this paper, we will only highlight a number of papers addressing Actor-Critic structure, one branch of ADP research that is closer in respect to the depth and width of our research. Readers who are interested in ADP should refer to \cite{ADP2009Survey} for a comprehensive review of ADP with respect to theoretical developments as well as application studies. The Actor-Critic structure was first proposed in \cite{ACD1983Sutton}, which suggested an optimal control of learning while improving. Actor-Critic structure implies two steps: an actor applies an action to the environment and receives feedback from a critic; Action improvement is then guided by the evaluation signal feedback. The decision environment feedback is recognized and reinforced after receiving feedback rewards with neural network (\cite{ADPNN1989Werbos},\cite{NNcontrol1991MSW} and \cite{NeuroDP1996}), probabilistic models \cite{RL1998} for bandit, Monto Carlo and Approximate MDP methods) and other stochastic models \cite{SLO2009Cao}. There are two main challenges in solving ADP: (1) curse of dimensionality: as the dimensions of state space and action space get extremely high, a large amount of information must be stored and it makes the computational cost grows explosively \cite{ADPCoD2009}; (2) Implicit form of objective functions: reward/cost function in dynamic control does not have an explicit form, which needs to be carefully approximated \cite{OLADP2013}. Powell introduced several parametric approximation methods to mitigate curse of dimensionality in \cite{ADPCoD2009} . Powell et al. in \cite{OLADP2013} proposed general dynamic control heuristics in ADP, including myopic control and lookahead control with different approximation schemes for cost function and decision environment transition probabilities, while decision environment is learned using local searching, regression or Bayesian methods with either offline or online fashion.

Our research is motivated by the current research gap in risk management literatures. Problem formulations in Section \ref{sec:problem} is supported by POMDP literatures, and heuristic solution algorithms are inspired by the ideas in ADP literatures. We studied some realistic issues in fraud control domain, and adapted the general POMDP models and ADP heuristics to fit the structure of fraud control problem. The model and algorithms proposed in this paper are not limited to the application of transaction fraud control, and can be easily extended to other fraud control and defense applications in finance, healthcare, electrical system, robotics, and homeland security.

\section{Problem Formulation}\label{sec:problem}
In this section, we rigorously formulate the dynamic control model assuming that the state information and the state transition information in the dynamic control model can be exactly characterized. However, the state information and the state transition probabilities in perfect state model, called \eqref{mod:base} in Section \ref{sec:perfect-model}, are not explicit, which need to be approximated from incomplete streaming data. Section \ref{sec:intractable-issue} discusses challenges in solving the dynamic model.

\subsection{Perfect State Dynamic Model}\label{sec:perfect-model}
We first focus and investigate the expected profit in transaction level, which are the building blocks of the control system. Let $s$, $m$ and $c$ denote risk score, profit margin and costs (cost of goods, manual review costs, chargeback fine, etc.) respectively. $s$ has a finite integral support $[\bar{s}]=\{0,1,...,\bar{s}-1,\bar{s}\}$ with upper bound $\bar{s}$, and $m\in\mathbb{R}$, $c\in\mathbb{R}$ are real numbers. According to system logistics shown in Figure \ref{fig:flow}, profits of approval ($app$), review ($rev$) and rejection ($rej$) of this transaction $w=(s,m,c)$ can be formulated as follow:
\begin{align*}
R_{app}(w)=&\delta_{w( \text{ Bank Auth. } \cap \text{ Non-fraud})}\cdot m - \delta_{w(\text{ Bank Auth. } \cap \text{ Fraud})}\cdot c\\
R_{rev}(w)=&\delta_{w(\text{ Bank Auth. } \cap \text{ MR App. } \cap \text{ Non-fraud})}\cdot m\\
& - \delta_{w( \text{ Bank Auth. } \cap \text{ MR App. } \cap \text{ Fraud})}\cdot c - \delta_{w( \text{ Bank Auth.})}\cdot c_0\\
R_{rej}(w)=&0
\end{align*}
where $c_0$ is unit labor cost for each manual review, and $\delta_{(\cdot)}$ is the indicator function, i.e. given event $H$,
\begin{align*}
\delta_{(H)}=\left\{
\begin{array}{ll}
1 & \text{if $H$ is true};\\
0 & \text{if $H$ is false}.
\end{array}
\right.
\end{align*}
Given the fact that risk score is a comprehensive evaluation of the risk level, which is estimated using thousands of transaction attributes, we assume that for any two transactions that have the same risk score $s$, i.e. $w=(s,m,c)$ and $w'=(s,m',c')$, the interactive effect of bank or MR are identical, which can be expressed in the mathematical form as,
\begin{align}
\pr(H\mid w)=\pr(H\mid s)=\pr(H\mid w')\label{equ:score-link}
\end{align}
With Eq.\eqref{equ:score-link}, the expected profit for each risk operation for transaction $w$ can be derived as
\begin{subequations}
\begin{align}
\mathbb{E}[R_{app}(w)]=&\pr(\text{Bank Auth. } \cap \text{ Non-fraud} \mid s )\cdot m\notag\\ & - \pr(\text{Bank Auth. } \cap \text{ Fraud} \mid s)\cdot c \notag\\
=& g_1(s)\cdot m - g_2(s)\cdot c\label{equ:exp-prof-app}\\
\mathbb{E}[R_{rev}(w)]=&\pr(\text{Bank Auth. } \cap \text{ MR App. } \cap \text{ Non-fraud} \mid s)\cdot m \notag\\
& - \pr(\text{Bank Auth. } \cap \text{ MR App. } \cap \text{ Fraud} \mid s)\cdot c \notag\\
& - \pr(\text{Bank Auth.}\mid s)\cdot c_0\notag\\
=& g_3(s)\cdot m-g_4(s)\cdot c -g_5(s)\cdot c_0\label{equ:exp-prof-rev}\\
\mathbb{E}[R_{rej}(w)]=&0.\label{equ:exp-prof-rej}
\end{align}
\end{subequations}
$g$-functions in Eq.\eqref{equ:exp-prof-app}-\eqref{equ:exp-prof-rej} are probabilities of different events given risk score $s$. $g$-function is short for gold function, whose values represent profit-related probabilities associated with different risk decisions.

We further delve into a realistic dynamic system, in which banks and MR decision behaviors are changing dynamically. We consider a discrete time dynamic control model with infinite time horizon $\mathcal{T}$. Let $\textbf{w}^{(t)}=\{w^{(t)}_1, w^{(t)}_2,...,w^{(t)}_{N(t)} \}$ be a set of transactions occurred during period $t\in\mathcal{T}$. Elements of this transaction set $w^{(t)}_j=(s^{(t)}_j,m^{(t)}_j,c^{(t)}_j)$  include the risk score $s^{(t)}_j$, margin $m^{(t)}_j$ and costs $c^{(t)}_j$ of this $j$th transaction in period $t$. Let $N^{(t)}$ be the total number of transactions occurred during period $t$, so $N^{(t)}$ is then a random variable. We can then formally define the dynamic control model as follow.
\begin{itemize}
\item State space: $\mathcal{S}=\{(g_1(s),g_2(s),g_3(s),g_4(s),g_5(s):s\in[\bar{s}]) \}$, which is a set of 5 $g$-functions values at all risk scores. In period $t$, the state can be expressed as $S^{(t)}=(g_1^{(t)}(s),g_2^{(t)}(s),g_3^{(t)}(s),g_4^{(t)}(s),g_5^{(t)}(s): s\in[\bar{s}])$.
\item Action space in period $t$: $\mathcal{A}^{(t)}=\{app, rev, rej\}^{N^{(t)}}$, which has $3^{N^{(t)}}$ feasible decision sequences. Let $\textbf{a}^{(t)}=\{a^{(t)}_1, a^{(t)}_2,...,a^{(t)}_{N(t)} \}$ be one feasible action sequence in period $t$, and for the $j$th transaction, risk control engine can choose action $a^{(t)}_j\in\{app,rev,rej\}$.
\item State transition probability matrix: $Q(\textbf{a})=[Q_{S,S'}(\textbf{a}):\forall S,S']$, where $Q_{S,S'}(\textbf{a})$ is the probability that system move from state $S$ to state $S'$ when taking action sequence $\textbf{a}$. We assume that $Q(\textbf{a})$ is fixed but implicit through out this paper.
\end{itemize}
Let $u(S^{(t)})$ be the reward-to-go function at the beginning of period $t$, then this stochastic dynamic model can be formulated with Bellman's equation as
\begin{align}
u(S^{(t)})=\max_{\textbf{a}^{(t)}\in \mathcal{A}^{(t)}}&\left\{\mathbb{E}\left[\sum_{j=1}^{N^{(t)}} R_{\textbf{a}^{(t)}_j}(w^{(t)}_j,S^{(t)})\right]\right.\notag\\ 
&\left.+\alpha\cdot \sum_{s^{(t+1)}}Q_{S^{(t)},S^{(t+1)}}(\textbf{a}^{(t)})\cdot u(S^{(t+1)}) \right\} \tag{Perfect} \label{mod:base}
\end{align}
where $\alpha$ is a discount factor of future rewards, and reward function can be formulated as
\begin{align*}
R_{\textbf{a}^{(t)}_j}(w^{(t)}_j,S^{(t)})=\left\{\begin{array}{ll}
g_1^{(t)}(s^{(t)}_j)\cdot m^{(t)}_j - g_2^{(t)}(s^{(t)}_j)\cdot c^{(t)}_j, & a^{(t)}_j=app\\
g_3^{(t)}(s^{(t)}_j)\cdot m^{(t)}_j-g_4^{(t)}(s^{(t)}_j)\cdot c^{(t)}_j -g_5^{(t)}(s)\cdot c_0, & a^{(t)}_j=rev\\
0, & a^{(t)}_j=rej
\end{array}\right.
\end{align*}
Throughout the entire paper, we assume that a finite number of transactions occurred in each period, and the reward of each transaction is bounded. Theorem \ref{thm:stationary-optimality} gives the condition that Model \eqref{mod:base} has a unique optimal solution.
\begin{theorem}\label{thm:stationary-optimality}
If (1) number of transaction occurred in each period is finite and margin/loss from each transaction is bounded, and (2) the arriving process of transactions is stationary, then there exists an optimal profit satisfying
\begin{align*}
u^*(S)=\max_{\textbf{a}}\left\{\mathbb{E}\left[\sum_{j=1}^{N} R_{\textbf{a}_j}(w_j,S)\right]+\alpha\cdot \sum_{S'}Q_{S,S'}(\textbf{a})\cdot u^*(S') \right\},
\end{align*}
and there is a unique solution to this equation.
\end{theorem}
Theorem \ref{thm:stationary-optimality} is guaranteed by contracting mapping argument and directly follows Theorem 6.2.3 and Theorem 6.2.5 from \cite{MDP1994}.
\subsection{Incomplete Information and Intractableness of \eqref{mod:base} Model}\label{sec:intractable-issue}
Although Theorem \ref{thm:stationary-optimality} provides solid guidance to find the optimal control strategy, there are several issues of implementing Model \ref{mod:base} in reality.
\begin{enumerate}[(1)]
\item Exact state information is unavailable: State information, i.e. $g$-functions, can only be inferred using partially mature data, since data maturity lead time is a latent random variable with range $[0, L]$. We have no way to obtain the true time point of maturity for each transaction until the transaction is eventually marked as a chargeback.  However, through analyzing the historical data we do have the knowledge that after $L$ periods of time the fraud status (having chargeback or not) should be all mature;
\item Reward functions are not entirely exact: Reward functions, $R(\cdot)$, are based on estimations of $g$-functions, and $\textbf{w}^{(t)}$ is not known a priori. Therefore, reward functions could vary due the different estimated $g$-functions and the different $\textbf{w}^{(t)}$;
\item Transition probability matrix $Q$ does not have an explicit form: State space $\mathcal{S}$ has extremely high dimension (five $g$-functions estimated at $(\bar{s}+1)$ risk scores); Action space $\mathcal{A}$ has exponential dimension that explosively increase as number of transactions increases ($\mathcal{A}^{(t)}=\{app, rev, rej\}^{N^{(t)}}$ has $3^{N^{(t)}}$ possible decision sequences).
\end{enumerate}
The lag of data maturity and the curse of dimensionality lead to the fact that Model \eqref{mod:base} is intractable. Thus we propose three approximate dynamic heuristics to obtain suboptimal control decisions. Details of these different control algorithms will be demonstrated in Section \ref{sec:control}. 

All dynamic control heuristics require a base module which utilizes incomplete information, such as the mature old data and the partially mature recent data, to infer future $g$-functions in these heuristic algorithms. Data mining results suggest that correlations exist between recent $l$ period's partially mature chargeback rate and bank/MR behavior patterns. This fact implies that we should track partially mature chargeback rate of transactions portfolio in period $t-l$, so that $g$-functions can be properly calibrated. This happens to have the same view with business intuitions in multi-party fraud control: If bank and MR learn that recently received transactions have high chargeback rate, they will become more conservative with their decision making by reducing the number of authorization/approval decisions to prevent more undesirable chargebacks. Two decision environment modules, Current Environment Inference (CEI) module and Future Environment Inference (FEI) module, are adopted from \cite{GFunctionEstimation2018TKDE}. Discussion of these two modules are out of the scope of the current paper, we suggest readers refer to \cite{GFunctionEstimation2018TKDE} for details of CEI and FEI modules. CEI and FEI utilize historical data to produce $g$-function estimations, which contribute to the data-driven property of our risk control framework.

\section{Dynamic Risk Control Algorithms}\label{sec:control}
In this section, we propose three different dynamic risk control algorithms: Naive, Myopic and Prospective control. Naive control is the simplest heuristic algorithm that only uses fully mature data before period $t-L$. Myopic control estimates the current decision environment using CEI module with both mature and immature data in period $t-l$. The most complex control model, Prospective control, further takes into account that current decision will influence not only the current profit but also the near future profit. Three models are demonstrated in Section \ref{sec:naive} - \ref{sec:prospective}.

\subsection{Naive control}\label{sec:naive}
Figure \ref{fig:naive} depicts decision flow of naive control. At the beginning of period $t$, decision engine uses mature data before period $t-L$ to estimate $g$-functions.
\begin{figure}[ht]
\centering
\includegraphics[scale=0.7]{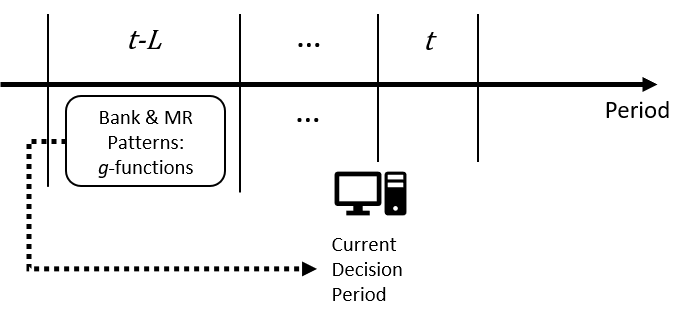}
\caption{Naive dynamic control}
\label{fig:naive}
\end{figure}

In period $t\in\mathcal{T}$, let $\hat S^{(t)}$=$(g_1^{(t-L)}(s)$, $g_2^{(t-L)}(s)$, $g_3^{(t-L)}(s)$, $g_4^{(t-L)}(s)$, $g_5^{(t-L)}(s):\forall s)$ be the estimated current state, and $\mathcal{A}^{(t)}=\{app, rev, rej\}^{N^{(t)}}$ be the action space of period $t$. Then feasible action sequence has a form of $a^{(t)}=\{a^{(t)}_1, a^{(t)}_2,...,a^{(t)}_{N^{(t)}} \}$, where $a^{(t)}_j\in\{app,rev,rej\}$. Naive model disregards the future effects. For $N^{(t)}$ transactions take place in period $t$, we need to solve the following model to get action sequence $a^{(t)*}$.
\begin{align}\label{mod:naive}
\max_{a^{(t)}\in \mathcal{A}^{(t)}}&\mathbb{E} \left[\sum_{j=1}^{N^{(t)}} \hat R_{a^{(t)}_j}(w^{(t)}_j)\right] \tag{Naive-t}\\
s.t.\quad & \mathbb{E}[\hat R_{app}(w_j^{(t)})]=g^{(t-L)}_1(s^{(t)})\cdot m - g^{(t-L)}_2(s^{(t)})\cdot c \notag\\
& \mathbb{E}[\hat R_{rev}(w_j^{(t)})]=g^{(t-L)}_3(s^{(t)})\cdot m - g^{(t-L)}_4(s^{(t)})\cdot c- g^{(t-L)}_5(s^{(t)})\cdot c_0 \notag\\
& \mathbb{E}[\hat R_{rej}(w_j^{(t)})]=0 \notag\\
&\mathcal{A}^{(t)}=\{app, rev, rej\}^{N^{(t)}}\notag
\end{align}
Naive control repeats this procedure for each period $t$. Theorem \ref{thm:greedy-naive} claims that \eqref{mod:naive} can be easily solved by greedily choosing the decision option that yields the highest expected reward for each incoming transaction. Details about Naive control policy is summarized in Algorithm \ref{alg:naive}.
\begin{theorem}\label{thm:greedy-naive}
Optimal action sequence $a^{(t)*}$ of \eqref{mod:naive} can be obtained by the greedy algorithm, i.e. for $w_j^{(t)}\in \textbf{w}^{(t)}$, sequentially set
\begin{align*}
a_j^{(t)*}=\arg\max_{a_j^{(t)}\in\{app,rev,rej\}} \mathbb{E}[\hat R_{a_j^{(t)}}(w_j^{(t)})]
\end{align*}
\end{theorem}
\begin{proof}
The rewards of different transactions are independent, and for period $t$, \eqref{mod:naive} can be decomposed into $N^{(t)}$ sub-maximization problems. Thus the greedy algorithm can solve \eqref{mod:naive} exactly.
\end{proof}

\begin{algorithm}[htbp]
\caption{Naive Dynamic Control}\label{alg:naive}
Repeat for period $t\in\mathcal{T}$:
\begin{algorithmic}[1]
\STATE Estimate $g_1^{(t-L)}(s),g_2^{(t-L)}(s),g_3^{(t-L)}(s),g_4^{(t-L)}(s)$ and $g_5^{(t-L)}(s)$ using all the data until the end of period $t-L$, let $\hat S^{(t)}=(g_1^{(t-L)}(s), g_2^{(t-L)}(s), g_3^{(t-L)}(s), g_4^{(t-L)}(s), g_5^{(t-L)}(s):\forall s)$;
\FOR {$j=1,2,...,N^{(t)}$}
\STATE $a_j^{(t)*}=\arg\max_{a_j^{(t)}\in\{app,rev,rej\}} \mathbb{E}[\hat R_{a_j^{(t)}}(w_j^{(t)})]$.
\ENDFOR
\end{algorithmic}
\end{algorithm}	

\subsection{Myopic control}\label{sec:myopic}
Figure \ref{fig:myopic} shows the decision flow of myopic control. \\
\begin{figure}[ht]
\centering
\includegraphics[scale=0.7]{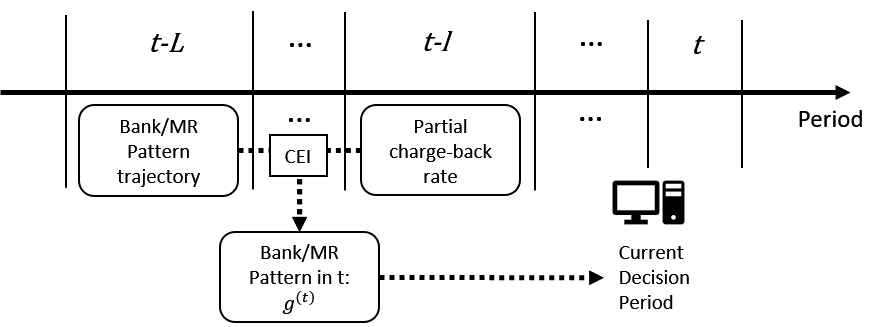}
\caption{Myopic dynamic control}
\label{fig:myopic}
\end{figure}

This control model is designed to resolve the pattern recognition lag issue due to the delay of data maturity. We adopt CEI module from \cite{GFunctionEstimation2018TKDE} to infer current period decision environments. Mathematically, CEI maps matured $g$-function trajectories ($g^{(t')}_1(s)$, $g^{(t')}_2(s)$, $g^{(t')}_3(s)$, $g^{(t')}_4(s)$, and $g^{(t')}_5(s)$: $t'\le t-L$) and partially mature chargeback rate $\rho_{PCB}^{(t-l)}$ to estimate $g$-functions ($\hat g^{(t)}_1(s), \hat g^{(t)}_2(s), \hat g^{(t)}_3(s), \hat g^{(t)}_4(s)$ and $\hat g^{(t)}_5(s)$) at current period, $t$.
\begin{align}
\begin{bmatrix}
\hat g^{(t)}_1(s)\\
\hat  g^{(t)}_2(s)\\
\hat  g^{(t)}_3(s)\\
\hat  g^{(t)}_4(s)\\
\hat  g^{(t)}_5(s)
\end{bmatrix} =
\begin{bmatrix}
\hat \Phi_1^{(t)} \left(s,g^{(t')}_1(s),\rho_{PCB}^{(t-l)}\right)\\
\hat \Phi_2^{(t)} \left(s,g^{(t')}_2(s),\rho_{PCB}^{(t-l)}\right)\\
\hat \Phi_3^{(t)} \left(s,g^{(t')}_3(s),\rho_{PCB}^{(t-l)}\right)\\
\hat \Phi_4^{(t)} \left(s,g^{(t')}_4(s),\rho_{PCB}^{(t-l)}\right)\\
\hat \Phi_5^{(t)} \left(s,g^{(t')}_5(s),\rho_{PCB}^{(t-l)}\right)
\end{bmatrix}\tag{CEI}
\end{align}
where ${\rho}^{t-l}_{PCB}$ is calculated by
$${\rho}^{t-l}_{PCB}=\frac{\left(\begin{array}{c}
	\text{\# of chargeback transactions}\\ \text{in week $t-l$ occurred before week $t$}
	\end{array}\right)}{\text{(\# of finally approved transactions in week $t-l$)}}.$$
Then for $N^{(t)}$ transactions occurred in period $t$, Myopic Dynamic Control model solves the following model to get action sequence $a^{(t)*}$.
\begin{align}\label{mod:myopic}
\max_{a^{(t)}\in \mathcal{A}^{(t)}}&\mathbb{E} \left[\sum_{j=1}^{N^{(t)}} \hat R_{a^{(t)}_j}(w^{(t)}_j)\right] \tag{Myopic-t}\\
s.t.\quad & \mathbb{E}[\hat R_{app}(w_j^{(t)})]=\hat g^{(t)}_1(s^{(t)})\cdot m - \hat g^{(t)}_2(s^{(t)})\cdot c \notag\\
& \mathbb{E}[\hat R_{rev}(w_j^{(t)})]=\hat g^{(t)}_3(s^{(t)})\cdot m - \hat g^{(t)}_4(s^{(t)})\cdot c - \hat g^{(t)}_5(s^{(t)})\cdot c_0 \notag\\
& \mathbb{E}[\hat R_{rej}(w_j^{(t)})]=0 \notag\\
&\mathcal{A}^{(t)}=\{app, rev, rej\}^{N^{(t)}}\notag
\end{align}
CEI module is updated at the beginning of each period and \eqref{mod:myopic} is solved during each period to provide optimal control actions. Theorem \ref{thm:myopic-greedy} provides theoretical guarantee that \eqref{mod:myopic} can be solved by the greedy method. Details of Myopic control policy is summarized in Algorithm \ref{alg:myopic}.
\begin{theorem}\label{thm:myopic-greedy}
The optimal action sequence $a^{(t)*}$ of \eqref{mod:myopic} can be obtained by the greedy algorithm, i.e. for $w_j^{(t)}\in \textbf{w}^{(t)}$, sequentially set
\begin{align*}
a_j^{(t)*}=\arg\max_{a_j^{(t)}\in\{app,rev,rej\}} \mathbb{E}[\hat R_{a_j^{(t)}}(w_j^{(t)})]
\end{align*}
\end{theorem}
The proof of Theorem \ref{thm:myopic-greedy} is similar with proof of Theorem \ref{thm:greedy-naive} and thus omitted.
\begin{algorithm}[htbp]
\caption{Myopic Dynamic Control}\label{alg:myopic}
Repeat for period $t\in\mathcal{T}$:
\begin{algorithmic}[1]
\STATE Calculate $g_1^{(t-L)}(s),g_2^{(t-L)}(s),g_3^{(t-L)}(s),g_4^{(t-L)}(s)$ and $g_5^{(t-L)}(s)$ using all the data until the end of period $t-L$, calculate $\rho^{(t-l)}_{PCB}$ using partially mature data in period $t-l$;
\STATE Estimate $\hat g_1^{(t)}(s)$, $\hat g_2^{(t)}(s)$, $\hat g_3^{(t)}(s)$, $\hat g_4^{(t)}(s)$ and $\hat g_5^{(t)}(s)$ using CEI module, and let $\hat S^{(t)}=(\hat g_1^{(t)}(s),\hat g_2^{(t)}(s),\hat g_3^{(t)}(s),\hat g_4^{(t)}(s):\forall s)$;
\FOR {$j=1,2,...,N^{(t)}$}
\STATE $a_j^{(t)*}=\arg\max_{a_j^{(t)}\in\{app,rev,rej\}} \mathbb{E}[\hat R_{a_j^{(t)}}(w_j^{(t)})]$.
\ENDFOR
\STATE Re-train and update CEI module.
\end{algorithmic}
\end{algorithm}

\subsection{Prospective control}\label{sec:prospective}
Figure \ref{fig:prospective} depicts decision flow of prospective control.
\begin{figure}[ht]
\centering
\includegraphics[scale=0.7]{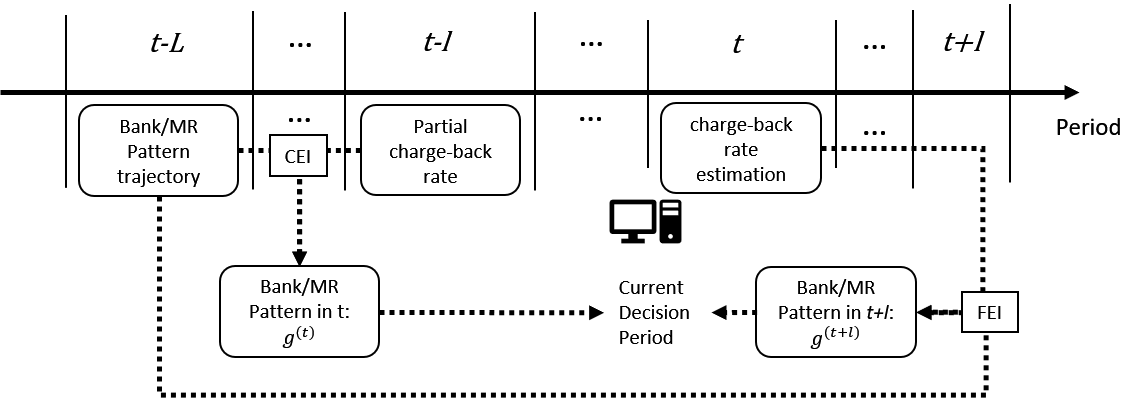}
\caption{Prospective dynamic control}
\label{fig:prospective}
\end{figure}
Prospective control model has a similar CEI module that can diminish pattern recognition lag. In addition, FEI module is adopted from \cite{GFunctionEstimation2018TKDE} to estimate future decision environment change due to the action taken at current period. These environments are characterized by the $g$-functions of period $t$ and $t+l$. Similar with Myopic control, in period $t\in\mathcal{T}$, we use the output of the CEI module as the state estimation, i.e. $\hat S^{(t)}=(\hat g_1^{(t)}(s),\hat g_2^{(t)}(s),\hat g_3^{(t)}(s),\hat g_4^{(t)}(s),\hat g_5^{(t)}(s):\forall s)$. Action space of period $t$ is still $\mathcal{A}^{(t)}=\{app, rev, rej\}^{N^{(t)}}$. While different from previous two control models, prospective control considers future effects caused by the current decisions: the action sequences will play a role on the behavior patterns of bank and MR in period $t+l$. For $N^{(t)}$ transactions occurred in period $t$, we need to solve the following model to get our action sequence $a^{(t)*}$.
\begin{align}\label{mod:prospective}
\max_{a^{(t)}\in \mathcal{A}^{(t)}}&\mathbb{E} \left[\sum_{j=1}^{N^{(t)}} \hat R_{a^{(t)}_j}(w^{(t)}_j)\right]+ \lambda\cdot \Delta \tag{Prospective-t}\\
s.t.\quad & \mathbb{E}[\hat R_{app}(w_j^{(t)})]=\hat g^{(t)}_1(s_j^{(t)})\cdot m_j^{(t)}- \hat g^{(t)}_2(s_j^{(t)})\cdot c_j^{(t)},\ \forall j  \notag\\
& \mathbb{E}[\hat R_{rev}(w_j^{(t)})]=\hat g^{(t)}_3(s_j^{(t)})\cdot m_j^{(t)}\notag - \hat g^{(t)}_4(s_j^{(t)})\cdot c_j^{(t)} - \hat g^{(t)}_5(s_j^{(t)})\cdot c_0,\ \forall j \notag\\
& \mathbb{E}[\hat R_{rej}(w_j^{(t)})]=0,\ \forall j \notag\\
&\mathcal{A}^{(t)}=\{app, rev, rej\}^{N^{(t)}}\notag
\end{align}
where $\lambda$ is a discount factor, and $\Delta$ is a reference future profit of period $t+l$. A reference sample from mature control group is bootstrapped from mature data set in order to provide reference future profit $\Delta$. Let this reference transaction set sample be $\tilde{\textbf{w}}^{(t+l)}=\{\tilde w_1^{(t+l)},\tilde w_2^{(t+l)},...,w_m^{(t+l)} \}$ with $m$ elements. FEI module includes two sub-procedures: 
\begin{enumerate}[(1)]
	\item  Calculate estimated chargeback rate of period $t$, $\rho^{(t)}_{CB}$: at a given time point during period $t$, suppose we have received $n'$ transaction request, and our decision action sequence is $(a^t_1,...,a^t_{n'})$, we can then estimate charge back rate of period $t$, 
	\begin{align}\label{equ:rho-estimation}
	\hat \rho^{(t)}_{CB}=\frac{1}{\sum_{j=1}^{n'}\delta_{(a^t_j\neq Rej.)}}\left(\sum_{j=1}^{n'} \hat g^t_2(s^t_j)\cdot\delta_{(a^t_j=App.)}+\sum_{j=1}^{n'} \hat g^t_4(s^t_j)\cdot\delta_{(a^t_j=Rev.)} \right)
	\end{align}
	where $\delta_{(\cdot)}$ is the indicator function.
	\item Predict future $g$-functions ($g^{(t+l)}_1(s)$, $g^{(t+l)}_2(s)$, $g^{(t+l)}_3(s)$, $g^{(t+l)}_4(s)$ and $g^{(t+l)}_5(s)$) with matured $g$-function trajectories ($g^{(t')}_1(s)$, $g^{(t')}_2(s)$, $g^{(t')}_3(s)$, $g^{(t')}_4(s)$, and $g^{(t')}_5(s)$: $t'\le t-L$) and estimate weekly full chargeback rate $\rho_{CB}^{(t)}$. FEI is trained with mature data and
	\begin{align}\label{equ:fei}
	\begin{bmatrix}
	\hat g^{(t)}_1(s)\\
	\hat  g^{(t)}_2(s)\\
	\hat  g^{(t)}_3(s)\\
	\hat  g^{(t)}_4(s)\\
	\hat  g^{(t)}_5(s)
	\end{bmatrix} =
	\begin{bmatrix}
	\hat \Psi_1^{(t)} \left(s,g^{(t')}_1(s),\hat \rho_{CB}^{(t)}\right)\\
	\hat \Psi_2^{(t)} \left(s,g^{(t')}_2(s),\hat \rho_{CB}^{(t)}\right)\\
	\hat \Psi_3^{(t)} \left(s,g^{(t')}_3(s),\hat \rho_{CB}^{(t)}\right)\\
	\hat \Psi_4^{(t)} \left(s,g^{(t')}_4(s),\hat \rho_{CB}^{(t)}\right)\\
	\hat \Psi_5^{(t)} \left(s,g^{(t')}_5(s),\hat \rho_{CB}^{(t)}\right)
	\end{bmatrix}\tag{FEI}
	\end{align}
\end{enumerate}

\eqref{mod:prospective} is hard to solve due to high dimension of $a^{(t)}$ and non-analytic form of $\Delta$. A similar real-time updated greedy heuristic is introduced to obtain a sub optimal solution for \eqref{mod:prospective}. This Real-time Greedy Heuristic (RGH) allows us to update estimation of $\hat \rho^{(t)}_{CB}$ on the fly and to adjust our strategy within period $t$. Figure \ref{fig:rgh} illustrates the logics of RGH within period $t$. \\
\begin{figure}[h]
\centering
\includegraphics[scale=0.8]{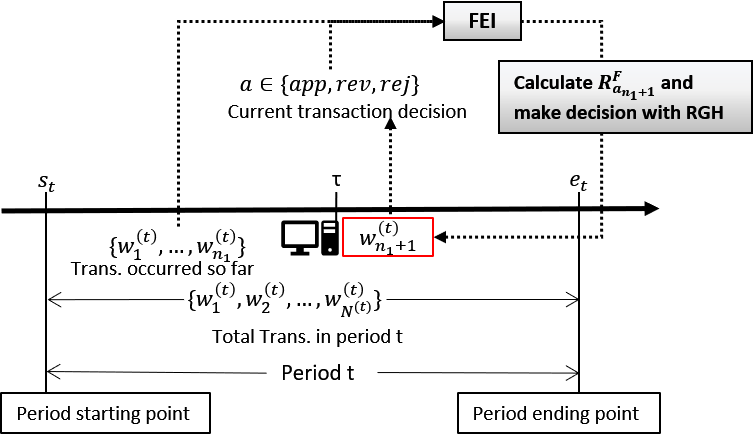}
\caption{Real-time greedy heuristic demonstration}\label{fig:rgh}
\end{figure}

Let time $\tau$ be a decision time point in period $t$ where transaction $w^{(t)}_{n_1+1}$ occurs and risk team needs to make decision either to approve, reject or manual review this transaction. Suppose from $s_t$, starting point of period $t$, to current decision point $\tau$, we have observed $n_1$ transactions. Hence, we can estimate the chargeback rate of period $t$, if we approve, review or reject $w^{(t)}_{n_1+1}$ using Eq. \eqref{equ:rho-tau}.
\begin{align}\label{equ:rho-tau}
\hat \rho^{(t)}_{CB}(\tau)=\frac{1}{\sum_{j=1}^{n_1+1}\delta_{(a^t_j\neq Rej.)}}\left(\sum_{j=1}^{n_1+1} \hat g^t_2(s^t_j)\cdot\delta_{(a^t_j=App.)}+\sum_{j=1}^{n_1+1} \hat g^t_4(s^t_j)\cdot\delta_{(a^t_j=Rev.)} \right)
\end{align}

We further estimate the expected reward of approval, review or rejection of $w^{(t)}_{n_1+1}$. Note that the future effect is first averaged to reward per transaction and then discounted by a factor of $\lambda$.
\begin{subequations}\label{equ:f-profit}
\begin{align}
&R^{F}_{app}(w^{(t)}_{n_1+1})=\mathbb{E}[\hat R_{app}(w^{(t)}_{n_1+1})]+\frac{\lambda}{m}\Delta_{\tau,app}\\
&R^{F}_{rev}(w^{(t)}_{n_1+1})=\mathbb{E}[\hat R_{rev}(w^{(t)}_{n_1+1})]+\frac{\lambda}{m}\Delta_{\tau,rev}\\
&R^{F}_{rej}(w^{(t)}_{n_1+1})=\mathbb{E}[\hat R_{rej}(w^{(t)}_{n_1+1})]+\frac{\lambda}{m}\Delta_{\tau,rej}
\end{align}
\end{subequations}
and for $a\in\{app,rev,rej\}$,
\begin{align}\label{equ:delta-estimation}
\Delta_{\tau,a}=\\
\max_{a^{(t+l)}\in \tilde{\mathcal{A}}^{(t+l)}}&\mathbb{E} \left[\sum_{k=1}^{m} \hat R_{a^{(t)}_k}(\tilde w^{(t+l)}_k)\right]\notag\\
s.t.\quad & \mathbb{E}[\hat R_{app}(\tilde w_k^{(t+l)})]=\hat g^{(t+l)}_1(\tilde s^{(t+l)}_k)\cdot\tilde  m^{(t+l)}_k - \hat g^{(t+l)}_2(\tilde s^{(t+l)}_k)\cdot\tilde  c^{(t+l)}_k, \notag\\
& \mathbb{E}[\hat R_{rev}(\tilde w_k^{(t+l)})]=\hat g^{(t+l)}_3(\tilde s^{(t+l)}_k)\cdot\tilde  m^{(t+l)}_k - \hat g^{(t+l)}_4(\tilde s^{(t+l)}_k)\cdot\tilde  c^{(t+l)}_k \notag\\&\qquad - \hat g^{(t+l)}_5(\tilde s^{(t+l)}_k)\cdot c_0,\notag\\
& \mathbb{E}[\hat R_{rej}(\tilde w_k^{(t+l)})]=0,\notag\\
& \tilde{\mathcal{A}}^{(t+2)}=\{app, rev, rej\}^{m}\notag
\end{align}
where $\hat \rho_{CB}^{(t)}(\tau)$ are calculated using Eq. \eqref{equ:rho-tau}, and $\hat g^{(t+l)}_{(\cdot)}$ is derived by \eqref{equ:fei}. RGH sequentially assigns action that has the largest prospective reward to each incoming transaction. For $w^{(t)}_j\in \textbf{w}^{(t)}$, we sequentially set
\begin{align}\label{equ:rgh}
a^{(t)*}_j=\arg\max_{a^{(t)}_j\in\{App.,Rev.,Rej.\}}R^{F}_{a^{(t)}_j}(w^{(t)}_j).\tag{Prospective-RGH}
\end{align}
Prospective control algorithm is summarized in Algorithm \ref{alg:prospective}.
\begin{algorithm}[htbp]
\caption{Prospective Dynamic Control}\label{alg:prospective}
Repeat for period $t\in\mathcal{T}$:
\begin{algorithmic}[1]
\STATE Calculate $g_1^{(t-L)}(s),g_2^{(t-L)}(s),g_3^{(t-L)}(s),g_4^{(t-L)}(s)$ and $g_5^{(t-L)}(s)$ using all the data until the end of period $t-L$, calculate $\rho^{(t-l)}_{PCB}$ using partially mature data in period $t-l$;
\STATE Estimate $\hat g_1^{(t)}(s)$, $\hat g_2^{(t)}(s)$, $\hat g_3^{(t)}(s)$, $\hat g_4^{(t)}(s)$ and $\hat g_5^{(t)}(s)$ using CEI module, and let $\hat S^{(t)}=(\hat g_1^{(t)}(s),\hat g_2^{(t)}(s),\hat g_3^{(t)}(s),\hat g_4^{(t)}(s):\forall s)$;
\STATE Initialize $\hat g_1^{(t+2)}(s)$, $\hat g_2^{(t+2)}(s)$, $\hat g_3^{(t+2)}(s)$, $\hat g_4^{(t+2)}(s)$ and $\hat g_5^{(t+2)}(s)$ by setting them equal to $\hat g_1^{(t)}(s)$, $\hat g_2^{(t)}(s)$, $\hat g_3^{(t)}(s)$, $\hat g_4^{(t)}(s)$ and $\hat g_5^{(t)}(s)$ respectively;
\FOR {$j=1,2,...,N^{(t)}$}
\STATE Calculate $\hat \rho_{CB}^{(t)}(\tau_j)$ for $a_j^{(t)}\in\{app,rev,rej\}$ using Eq.\eqref{equ:rho-tau};
\STATE Calculate future reference profits $\Delta_{\tau,a}$ using Eq.\eqref{equ:delta-estimation};
\STATE Calculate prospective profits $\hat R^F_{a_j^{(t)}}(w^{(t)}_j)$ and set $$a_j^{(t)*}=\arg\max_{a_j^{(t)}\in\{app,rev,rej\}} \mathbb{E}[\hat R^F_{a_j^{(t)}}(w^{(t)}_j)].$$
\STATE Update chargeback rate estimation $\hat \rho^{(t)}_{CB}$ using Eq.\eqref{equ:rho-estimation}, and calculate $\hat g_1^{(t+2)}(s)$, $\hat g_2^{(t+2)}(s)$, $\hat g_3^{(t+2)}(s)$, $\hat g_4^{(t+2)}(s)$ and $\hat g_5^{(t+2)}(s)$.
\ENDFOR
\STATE Re-train and update CEI and FEI modules.
\end{algorithmic}
\end{algorithm}

\section{Field Tests on Microsoft E-commerce}\label{sec:test}
Field tests were conducted to exam the performances of these three dynamic models. Testing dataset was extracted from a sub-unit of Microsoft E-commerce business. We sample no more than 3\% of total transactions as the testing data set. For transactions in the testing set, we recorded decisions in our database while we flipped all final rejected transactions to final approval, so that we could obtain unbiased chargeback signals for model training and profit calculation. We set the length of the testing period to one week and tested all dynamic control model paralleling with current Microsoft inline decision engine. 

Our data indicated that maximum  lead time for the data maturity was $L=12$, and the recent partially mature reference time was $l=2$. The testing time window is 14 weeks, and the bank and MR decisions for each transaction are kept identical for different control methods to ensure apple-to-apple comparison. The historical data continued to be maturing while the testing time moved forward. For Naive and Myopic controls, the risk decision engine was updated weekly ($g$-functions and (CEI) module was retrained at the beginning of each period). For Prospective control, the risk decision engine refreshed the belief of current $g$-functions, (CEI) module and (FEI) module once a week, while estimations of current week chargeback rate and future $g$-functions in real-time were updated. Due to the Microsoft's confidentiality requirements, the name of the E-commerce sub-unit is muted, and this section only includes the summarized feature values that were aggregated over a 14-week of transaction period to demonstrate the usability of Naive, Myopic and Prospective control models. The discount factor $\lambda$ in Prospective control model was tuned using $K$-fold cross validation at the beginning of the testing and is a fixed valued, $0.12$, throughout the 14 week testing periods.
\begin{figure}[htbp]%
\centering
\subfloat[\# of approve]{%
\label{fig:n-app}%
\includegraphics[scale=0.5]{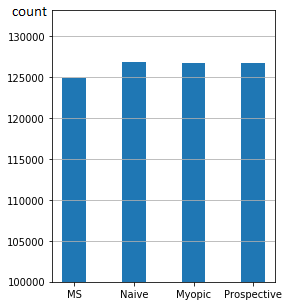}}%
\subfloat[\# of review]{%
\label{fig:n-rev}%
\includegraphics[scale=0.5]{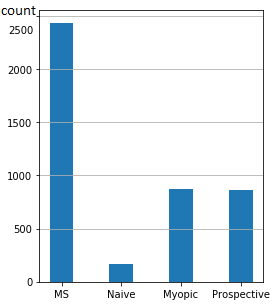}}%
\subfloat[\# of reject]{%
\label{fig:n-rej}%
\includegraphics[scale=0.5]{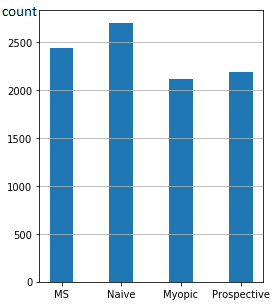}}%
\caption{Counts of Different Types of Risk Decisions Madevolume}\label{fig:n-decision}
\end{figure}

We first studied numbers of different risk control operations (approve, review, and reject) out of total testing transactions. Figure \ref{fig:n-decision} summarizes counts of different risk control decisions made by the current Microsoft's decision engine and three proposed dynamic control engines. Over the 14 weeks testing period, the dynamic control engines gradually captured the decision accuracy of the manual review group. All of the three models learned the fact that manual review agents overly rejected non-fraud transactions, and thus started to cut off the volume of transactions submitted to the manual review team. Figure \ref{fig:n-app} suggests that all three dynamic models approved more transactions than the current decision engine did. We can see later that these dynamic models also enhanced decision accuracy significantly in Figure \ref{fig:fn-fp-mr}: the dynamic control models not only approved more non-fraud transactions but also approved fewer fraud transactions. All three models suggest sending fewer transactions for manual review. Naive control aggressively decreased review volume to only 10\% of the review volume suggested by current Microsoft's decision engine, while Myopic and Prospective control mildly decreased review volume to roughly 30\% of the original volume. As for the decision of rejection, Naive control increased rejection volume by about 12\%, while Myopic and Prospective control decreased rejection volume by 12.5\% and 9\% respectively. We can also observe the fact that Myopic and Prospective control models again enhanced decision accuracy in Figure \ref{fig:fn-fp-mr} by rejecting much fewer non-fraud transactions but more fraud transactions.

Numbers of performance measures were used to validate the decision quality of a risk control engine. First, we investigated the decision quality by comparing the losses caused by wrong decisions. Two common performance metrics for this are false negative (FN) loss and false positive (FP) loss. FN loss measures the total loss of approving fraud transactions  (wrongly approval), which consists cost of goods and all related fees of chargeback. On the other hand, FP loss measures the total loss of rejecting non-fraud transactions (wrongly rejection), and it includes all the margins that should have been but not earned. We then checked the manual review (MR) cost, which is the total labor cost of the human review team. We found that when the risk engine submitted transactions that included fewer frauds (true negative: rightful approval) to the manual review teams,  manual review teams tended to have a much more difficult time to make accurate risk decisions since fraud patterns are less massive and recognizable. Therefore, with more transactions sent to manual review teams, not only more labor costs will arise, but the decision accuracy instability will likely to increase. Figure \ref{fig:fn-fp-mr} summarizes aggregated improvement on FN loss, FP loss and MR cost on the selected testing data set.
\begin{figure}[htbp]%
\centering
\subfloat[FN loss difference in \%]{%
\label{fig:fn}%
\includegraphics[scale=0.55]{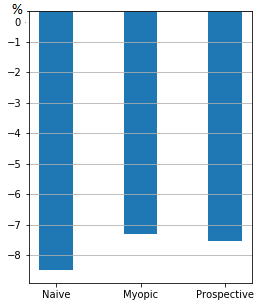}}%
\subfloat[FP loss difference in \%]{%
\label{fig:fp}%
\includegraphics[scale=0.55]{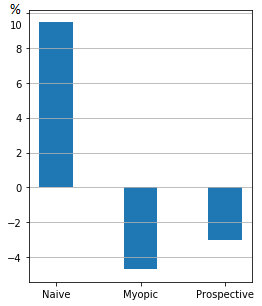}}%
\subfloat[MR cost difference in \%]{%
\label{fig:mr}%
\includegraphics[scale=0.55]{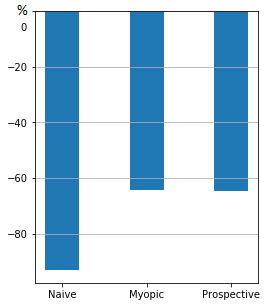}}%
\caption{Aggregated improvement on FN loss, FP loss and MR cost}\label{fig:fn-fp-mr}
\end{figure}
Figure \ref{fig:fn} shows the fact that all three dynamic control methods made better "approval" decisions by producing fewer FN losses. Naive, Myopic and Prospective control model decreases FN loss by 8.48\%, 7.32\%, and 7.55\% respectively. Figure \ref{fig:fp} suggests that Naive control model is relatively aggressive which rejected more non-fraud transactions and yielded 9.49\% more FP losses. Meanwhile, Myopic and Prospective control mildly decrease FP loss by 4.73\% and 3.05\% respectively, and these two dynamic control methods make more correct rejections. As mentioned earlier, all dynamic decision engine found that MR had limited accuracy in detecting fraud. In this way, Naive, Myopic and Prospective control model deceased transactions submit for review by 93.0\%, 64.2\%, and 64.7\%.

Second, we compare the differences of total profits and total chargeback rates among three dynamic control methods and current Microsoft's risk control method. Providing higher profit is the ultimate goal for business operations. While on the other hand, risk control team also needs to ensure the new dynamic control methods do not escalate the chargeback rate for merchants. We need to ensure that proposed dynamic control methods can produce higher profit but not increase (or even lower) the chargeback rate.\\
\begin{table}[htbp]
	\centering\caption{Aggregated performance improvements in profit and chargeback rate}
	\begin{tabular}{|c|c|c|c|}
		\hline
		& Naive & Myopic & Prospective \\ \hline
\begin{tabular}[c]{@{}c@{}}14 week aggregated\\ improvement on\\ testing set (\$)\end{tabular} & + \$ 79,962 & + \$ 97,863 & + \$ 96,693 \\ \hline
\begin{tabular}[c]{@{}c@{}}Estimated annual\\ improvement on\\ selected sub-department (\$)\end{tabular} & + \$ 9,900,071 & + \$ 12,116,318 & + \$ 11,971,568 \\ \hline
\begin{tabular}[c]{@{}c@{}}Relative chargeback\\ rate difference (\%)\end{tabular} & -0.72\% & -1.64\% & -2.98\% \\ \hline
\end{tabular}
	\label{tbl:profit-cb}
\end{table}\\
Table \ref{tbl:profit-cb} summarizes insights of improvements in overall profit and chargeback rate. The first row of Table \ref{tbl:profit-cb} includes profit improvement on the testing set calculated by
$(\text{TotalProfit}_{(\text{Dynamic})}-\text{TotalProfit}_{(\text{Microsoft})})$.
The second row extrapolates total profit from training set to an estimated annual improvement on the selected sub-unit. The third row reports the relative differences in proportion  on chargeback rates,  calculated by
\[\dfrac{\text{chargeback rate}_{\text{(Dynamic)}}-\text{chargeback rate}_{\text{(Microsoft)}}}{\text{chargeback rate}_{\text{(Microsoft)}}}. \]
Over the 14 week testing period, Naive control contributed  \$79,962 more on the testing portfolio while maintained a similar chargeback rate with current Microsoft risk decision engine had. Naive control decreased chargeback rate slightly by only 0.72\% of Microsoft's current chargeback rate. Myopic control contributed to the largest profit improvement for \$97,863 on the testing set. Meanwhile, Myopic control decreased chargeback rate relatively for 1.64\%. Finally, for Prospective control, it produced \$96,693 more profit on the testing transaction set, while provided the largest improvement on chargeback rate by decreasing chargeback rate by 2.98\%. The estimated annual improvements for Naive, Myopic and Prospective control on selected sub-unit were \$ 9,900,071, \$ 12,116,318, and \$ 11,971,568 respectively by extrapolation.

We conclude this section with a few business takeaways. We have seen that all three models have potentials for significantly improving company profit while slightly decreasing chargeback rates. All three dynamic models enhanced decision qualities by decreasing FN losses, FP losses and MR costs. Although Naive control model performed relatively aggressive in rejecting transactions, Myopic and Prospective control made better rejection decisions by rejecting fewer non-fraud transactions. All three dynamic methods had great performance with approving more non-fraud transactions and rejecting more fraud transactions. Artificial intelligence modules in these dynamic control models were well developed, and outperformed human review agents one most of the fraud decisions. Manual review volumes decreased as expected, and hence MR labor costs were reduced significantly.

\section{Conclusion and Future Study}\label{sec:conclusion}
To minimize ad hoc human-made decision, and improve the accuracy and robustness of the risk decision making, we investigated how to reach the optimal action when if complete information is available. We defined our problem rigorously, characterized all profit related components in the current system and investigated decision interactions between three different decision-making parties.  We acknowledged the fact that perfect information is unavailable in reality and thus we designed three data-driven dynamic optimal control models, Naive control, Myopic control, and Prospective control. These control models are 100\% data-driven and self-trained/adapted in a real-time manner. As demonstrated, these dynamic control models helped increase the profit significantly by minimizing false negative loss, false positive loss, and manual review costs by employing incomplete information, including long-term and short-term mature and partially-mature data. Meanwhile, the proposed control models also slightly lowered chargeback rates as desired. The field test on sub-unit of Microsoft E-commerce suggested that the discriminative dynamic control models had better fraud detection performance than the current general score cut-off control.

The research proposed in this paper can contribute greatly to both theoretical and applied research on fraud detection for the systems that have problems with incomplete information and decision looping effect due to multiple decision parties.  Its application is not limited to financial risk systems, but can also be used for application and research in cyber-security, homeland security, contagion disease screens etc.. Our future research will include information sharing and information fusion. We will extend this current research to more complex and realistic settings, where information sources are shared at different levels among different risk control decision parties.

\section*{Acknowledgments and Funding Sources}
This research was supported by Microsoft, Redmond, WA. The authors are thankful to researchers and members from Microsoft Knowledge and Growth group for providing data and their knowledge of the system.

\bibliography{mybibfile}

\end{document}